\documentclass[11pt]{article}
\input{macros}
\usepackage{mathrsfs}
\usepackage{todonotes}
\usepackage{algpseudocode}
\allowdisplaybreaks

\definecolor{revisionred}{RGB}{190,0,0}

\newif\ifshowrevisions
\showrevisionstrue 

\ifshowrevisions
    \newcommand{\rev}[1]{{\color{black}#1}}
    \newenvironment{revision}{\begingroup\color{black}}{\endgroup}
\else
    \newcommand{\rev}[1]{#1}
    \newenvironment{revision}{}{}
\fi

\showrevisionstrue

\begin{document}

\title{Minimax Optimality of Score-Entropy Discrete Diffusion}
\author{%
	Cholyeon Cho\thanks{Department of Computer Science,  The University of Texas at Austin; 
	} 
	\and
Yuchen Wu\thanks{School of Operations Research and Information Engineering, Cornell University.  
}}
\date{\today}
\maketitle
\begin{abstract}
Discrete diffusion models have demonstrated strong performance across a range of datasets, including natural language data and graph-structured data.
Among many variants, score-entropy discrete diffusion (SEDD) has achieved particularly strong empirical results. 
In SEDD, new samples are generated by iteratively evaluating a sequence of \emph{concrete score functions}, which are learned by minimizing a \emph{score-entropy loss}.

\rev{While much of the prior theoretical literature on discrete diffusion has focused on the sampling efficiency of SEDD under the assumption of small score estimation error, recent work has begun to investigate the finite-sample properties of score estimation itself. 
In this work, we take a different route by investigating the fundamental statistical limits of concrete score estimation. 
We focus on uniform and masking discrete diffusions, two of the most widely adopted discrete diffusion models. 
We establish a minimax lower bound under the score-entropy loss, and propose an MLE-based thresholding estimator that matches this lower bound up to constant and polylogarithmic factors that depend on neighboring density ratios.}
We further show that, for \rev{any target distribution}, this density ratio is naturally controlled under both uniform and masking discrete diffusion models, yielding nearly matching minimax lower and upper bounds for the aggregated score estimation error.
Our results imply that, with appropriate initialization and discretization, SEDD can achieve nearly optimal minimax sample complexity, as measured by the KL divergence between the target and generated distributions.
\end{abstract}

\section{Introduction}

Originally motivated by ideas from thermodynamics \cite{sohldickstein2015deep}, \emph{score-based diffusion models} have emerged as a powerful class of generative models, achieving state-of-the-art performance across a wide range of applications \cite{song2019generative,ho2020denoising,song2020denoising,rombach2022highresolution,saharia2022photorealistic,watson2023novo,yang2023diffusion,croitoru2023diffusion,abramson2024accurate}.
Conceptually, a diffusion model consists of two complementary processes: (1) a \emph{forward process} that gradually diffuses the target distribution into noise, and (2) a learned \emph{reverse process} that reconstructs the target distribution from noise. 
Sample generation is attained through the reverse process, which produces samples via an iterative denoising procedure driven by a sequence of time-dependent \emph{score functions}. 
These score functions capture essential information about the target distribution, and guide the progressive transformation of noise into structured data.
In practice, these functions are typically learned from data via a \emph{score matching} procedure \cite{hyvarinen2005estimation}.

Despite the remarkable success of diffusion models in continuous data domains such as images and audio, extending them to applications involving inherently discrete data such as text and categorical variables has long remained challenging. 
In these scenarios, \emph{discrete diffusion models}, which adapt the diffusion framework to discrete state spaces, have shown advantages over their continuous counterparts \cite{austin2021structured,vignac2023digress,shi2024simplified,lou2024discrete,alakhdar2024diffusion,sahoo2024simple,ou2025your}. 
Similar to continuous diffusion models, a discrete diffusion model also consists of a forward and a reverse process. 
However, unlike continuous diffusion models, which construct these processes by directly adding and removing noise in a continuous space, discrete diffusion models formulate both the forward and reverse processes as continuous-time Markov chains (CTMC) over a discrete state space $\mathcal{X}$ \cite{campbell2022continuous,benton2024denoising}.
These CTMCs are governed by a chosen \emph{transition rate matrix}, and induce a time-indexed family of distributions $\{p_t\}_{0 \leq t \leq T}$ over $\mathcal{X}$, where $p_0$ denotes the target distribution.
To construct the discrete diffusion reverse process and thereby enable sample generation, one needs to learn the associated \emph{concrete score functions}, defined by the collection of density ratios $\{p_t(y) / p_t(x)\}_{x, y \in \mathcal{X}}$ \cite{meng2022concrete}.

Among the various approaches proposed for learning these concrete score functions, score-entropy discrete diffusion (SEDD) has emerged as a particularly promising method, challenging the long-standing dominance of autoregressive language models \cite{lou2024discrete}. 
Specifically, SEDD learns the density ratios by minimizing a \emph{score-entropy loss}, leading to a training procedure that is different from the score matching framework used in continuous diffusion models. 
Specifically, for a distribution $p$ and an estimator $s_{\theta}(y, x)$ of the density ratio $p(y)/p(x)$, the population version of the score-entropy loss is given by 
\begin{align}
\label{eqn:SE_objective}
	\mathcal{L}_{\rm SE} = \E_{x \sim p} \Big[ \sum_{y \neq x} w_{yx}\, \frac{p(y)}{p(x)} D\big( s_{\theta}(y, x), \,  \frac{p(y)}{p(x)} \big) \Big], 
\end{align}
where $w_{yx} \geq 0$ is determined by the transition rate matrix, and $D(a, b) = a / b - 1 - \log (a / b)$ is the Bregman divergence.
Objective \eqref{eqn:SE_objective} can be optimized by approximating the expectation with an empirically tractable denoising score matching loss computed from the training samples. 
See \cite{lou2024discrete} for more details. 
Given a score estimate, the reverse dynamics can be simulated using various sampling techniques, including $\tau$-leaping and its variants \cite{gillespie2001approximate}, the uniformization technique \cite{jensen1953markoff}, Gillespie's method \cite{gillespie1976general}, and Euler method.

The success of SEDD has inspired a line of research seeking to build its theoretical foundations \cite{chen2025convergence,zhang2025convergence,liang2025absorb,ren2025how,liang2025discrete,pham2025discrete,dmitriev2026efficient,liang2026sharp}. 
Equipped with various sampling techniques, these works derive upper bounds on the Kullback-Leibler (KL) divergence between the target distribution and the output distribution, which roughly speaking can be decomposed into three terms corresponding to different sources of error: 
\begin{align}
\label{eq:three-terms}
	\KL(p_{\rm target} \parallel p_{\rm output}) \lesssim \,\cE_{\rm init} + \cE_{\rm dis} + \cE_{\rm SE}.  
\end{align}
\rev{See \cite{ren2025how} for a rigorous derivation of this error decomposition for both $\tau$-leaping and uniformization.}
We now explain each of the three terms above:
\begin{enumerate}
	\item[(i)] \rev{The term $\cE_{\rm init}$ represents the initialization error of the reverse CTMC. It measures the discrepancy between the terminal distribution $p_T$ of the forward process and the reference distribution used to initialize the reverse process. The time horizon $T$ is a user-specified parameter. $\cE_{\rm init}$ typically decays exponentially in $T$ and can therefore be controlled by choosing a moderately large time horizon. See, for instance, \cite[Lemma~7]{dmitriev2026efficient} for more details.}   
	\item[(ii)] \rev{The term $\cE_{\rm dis}$ denotes the discretization error incurred when the reverse continuous-time Markov chain is approximated by a finite-step sampling algorithm. It is defined by considering the sampling algorithm with access to the exact score functions and is therefore independent of the score estimation error. The discretization error depends on the specific sampling algorithm and can typically be made small by using a sufficiently fine discretization.}
	\item[(iii)] The final term $\cE_{\rm SE}$ quantifies the quality of score estimation and takes the form of a weighted average of the score-entropy loss in \cref{eqn:SE_objective} over different time points. \rev{It is the statistical error incurred when the true time-dependent score functions are replaced by estimates constructed from finitely many training observations.} See Eq.~\eqref{eq:aggregated-SE} for an explicit definition of $\cE_{\rm SE}$.  
\end{enumerate} 

\rev{Previous theoretical work on discrete diffusion models has largely focused on controlling sampling and discretization errors while assuming access to score estimates with sufficiently small estimation error. More recent work has begun to investigate the finite-sample properties and sample complexity of score estimation, including analyses of neural-network and discrete-flow-matching estimators \cite{srikanth2026discrete,wakasugi2025state,su2025theoretical,wan2025error}. Our contribution is complementary: we characterize the minimax statistical complexity of concrete-score estimation under the score-entropy loss and construct a simple estimator that nearly attains this minimax rate.}

In particular, we make the following contributions: 
\begin{itemize}
	\item[--] \emph{Sample complexity for estimating the concrete score function. }
We establish a minimax lower bound on the error of estimating a single score function under the score-entropy loss for a class of distributions with bounded neighboring density ratios. 
We further propose an MLE-based thresholding algorithm that matches this lower bound up to constant and poly-logarithmic factors. 
We note that this problem is high-dimensional and intrinsically challenging, as the number of states is large and may be comparable to the sample size. 
Related problems have been extensively studied in the discrete distribution estimation literature, see \cite{devroye2001combinatorial} for more details.
Our result can thus be viewed as an application of discrete distribution estimation techniques in the diffusion model setting. 
\item[--] \emph{Minimax optimality of SEDD. } 
Leveraging the minimax lower and upper bounds established here, we show that minimizing the score-entropy loss (together with sufficiently fine discretization and accurate initialization) yields a near-optimal sampling error, as measured by the KL divergence between the target and output distributions. These results provide partial justification for using the score-entropy loss to learn concrete score functions, demonstrating the near-optimality of SEDD. 
\end{itemize}

\subsection{Organization of the paper}

We present preliminaries on discrete diffusion models in Section~\ref{sec:discrete_diffusion}, followed by our main results in Section~\ref{sec:main}. In Section~\ref{sec:experiments}, we provide numerical experiments supporting our findings, and conclude with a discussion of our contributions in Section~\ref{sec:discussion}.

\subsection{Notation}

For $S, d \in \N_+$, we define the set $[S] = \{1, \dots, S\}$, and denote the $d$-fold Cartesian product of $[S]$ by $[S]^d$.
For $x, y \in [S]^d$, we denote by $\Ham(x, y)$ the Hamming distance between $x$ and $y$. 
For $r \in \R$, we define $\lceil r \rceil$ as the largest integer less than or equal to $r$, and define $\lfloor r \rfloor$ as the smallest integer greater than or equal to $r$.

\section{Preliminaries on discrete diffusion models}
\label{sec:discrete_diffusion}

In this section, we present preliminaries on discrete diffusion models.	
As in many prior works, we consider a probability distribution $p_0$ over a $d$-dimensional token space $[S]^d$, where each token is drawn from a vocabulary of size $S$. 
We discuss two major types of discrete diffusion models, \emph{uniform discrete diffusion} and \emph{masking discrete diffusion}, both are defined through Markov processes over a state space $\mathcal{X}$.
For uniform discrete diffusion, we set $\mathcal{X} = [S]^d$, 
whereas for masking discrete diffusion we set $\mathcal{X} = (\,[S] \cup \{\mask\}\,)^d$. Here, $\mask$ denotes a special symbol outside the regular vocabulary set $[S]$.

\subsection{Forward and reverse CTMCs}
\label{sec:forward-reverse-CTMC}

The forward process in a discrete diffusion model is defined as a CTMC initialized at the target distribution $p_0$, and is governed by a sequence of \emph{transition rate matrices} $\{\overrightarrow{Q}_t\}_{0 \le t \le T} \subseteq \R^{|\mathcal{X}| \times |\mathcal{X}|} $. 
These transition rate matrices satisfy the following properties: 
(1) $\overrightarrow{Q}_t(x, y) \geq 0$ for all $x \neq y$ and $x, y \in \mathcal{X}$; (2) $\overrightarrow{Q}_t(x, x) = - \sum_{y \neq x} \overrightarrow{Q}_t(x, y)$ for all $x \in \mathcal{X}$.
With these rate matrices, the corresponding CTMC $(\overrightarrow{x}_t)_{0 \leq t \leq T}$ is a right-continuous Markov process that satisfies for all $0\leq t < T$ and $x, y \in \mathcal{X}$,
\begin{align}
	\label{eq:forward}
		\P(\overrightarrow{x}_{t + \delta} = y \mid \overrightarrow{x}_t = x) = \mathbbm{1}\{x = y\} + \overrightarrow{Q}_t(x, y) \delta + o(\delta), \qquad \delta \to 0^+. 
\end{align}  
Roughly speaking, $\overrightarrow{Q}_t(x, y)$ characterizes how fast state $x$ transitions into state $y$ at time $t$. 

We assume $\overrightarrow{x}_0 \sim p_0$, and denote by $p_t$ the distribution of $x_t$ for $0 \leq t \leq T$. 
The distributions $(p_t)_{0 \leq t \leq T}$ can be viewed as vectors in $\R^{|\mathcal{X}|}$, and they satisfy the following Kolmogorov equation:  
\begin{align*}
    \frac{\dd p_t}{\dd t} = \overrightarrow{Q}_t^{\top} p_t, \qquad 0 \leq t \leq T. 
\end{align*}
The reverse process associated with process \eqref{eq:forward} is also a CTMC with transition rate matrices $\{\overleftarrow{Q}_t\}_{0 \leq t \leq T}$, defined as follows: 
\begin{align}
\label{eq:reverse-rate-matrix}
\begin{split}
    & \overleftarrow{Q}_t(x, y) = \overrightarrow{Q}_{T - t}(y, x) \cdot \frac{p_{T - t}(y)}{p_{T - t}(x)}\qquad  \forall x \neq y \mbox{ and } x, y \in \mathcal{X}, \\
    & \overleftarrow{Q}_t(x, x) = - \sum_{y \neq x} \overleftarrow{Q}_t(x, y). 
\end{split}
\end{align}
We denote the reverse CTMC by $\{\overleftarrow{x}_t\}_{0 \leq t \leq T}$. Note that if $\overleftarrow{x}_0 \overset{d}{=} \overrightarrow{x}_T \sim p_T$, then  $\overleftarrow{x}_t \overset{d}{=} \overrightarrow{x}_{T - t} \sim p_{T - t}$ for all $0 \leq t \leq T$ \cite{campbell2022continuous}.

To simplify computation, in discrete diffusion models we typically assume that the forward CTMC evolves independently and homogeneously across dimensions at the token level. 
Specifically, we assume that for $0 \leq t \leq T$, 
\begin{align*}
    \overrightarrow{Q}_t(x, y) = \left\{ \begin{array}{ll}
        Q^{\rm token}(x^i, y^i) & \mbox{if }\,\, \Ham(x, y) = 1 \mbox{ and }x^i \neq y^i,  \\
        0 & \mbox{otherwise}. 
    \end{array}\right. 
\end{align*}
In the above display, $Q^{\rm token}$ denotes a time-homogeneous token-wise transition rate matrix.
Two commonly used choices of such rate matrices are the \emph{uniform rate matrix} and the \emph{absorbing rate matrix}, which correspond respectively to uniform discrete diffusion and masking discrete diffusion. We define these two rate matrices below.
\begin{itemize}
	\item \textbf{Uniform rate matrix.} We set $Q^{\rm token} \in \R^S \times \R^S$, and for $s_1, s_2 \in [S]$ with $s_1 \neq s_2$ we define 
	\begin{align}
	\label{eq:uniform}
		Q^{\rm token}(s_1, s_2) = 1 / S. 
	\end{align}
	\item \textbf{Absorbing rate matrix.} In this case, we set $Q^{\rm token} \in \R^{S + 1} \times \R^{S + 1}$, and for $s_1, s_1 \in [S] \cup \{\mask\}$ with $s_1 \neq s_2$ we define 
	\begin{align}
	\label{eq:absorbing}
		Q^{\rm token}(s_1, s_2) = \mathbbm{1} \{s_1 \neq \mask, \, s_2 = \mask\}. 
	\end{align}
\end{itemize}

\subsection{Sampling with discrete diffusion models}

To generate new samples from $p_0$ using a discrete diffusion model, one needs to construct a Markov chain that approximates the reverse CTMC $\{\overleftarrow{x}_t\}_{0 \leq t \leq T}$ defined above.
As shown in \cref{eq:reverse-rate-matrix}, the rate matrices of the reverse CTMC depend on the concrete score functions $\{p_t(y) / p_t(x)\}_{x, y \in \mathcal{X}}$.    
In practice, these scores are unknown and must be estimated from data, which, as noted earlier, is typically done by minimizing the score-entropy loss \eqref{eqn:SE_objective} with $w_{yx} = \overrightarrow{Q}_t(y, x)$.

In addition to estimating the score functions, we typically cannot sample exactly from the theoretically correct initial distribution $p_T$. 
Therefore, we instead choose an alternative, easy-to-sample distribution that closely approximates $p_T$ to initialize the reverse CTMC. 
For uniform discrete diffusion, this can be the uniform distribution over $\mathcal{X}$.
For masking discrete diffusion, we initialize with a Dirac measure at $(\mask)^{\otimes d}$.

Beyond score estimation and approximate initialization, we must also discretize the reverse CTMC to obtain a practical algorithm for discrete diffusion models. The corresponding algorithms are already summarized in the introduction. In practice, we further early stop the reverse process at time $T - \delta$ for some small positive $\delta$ to prevent numerical instability.

\section{Main results}
\label{sec:main}

For a discretization scheme $0 = t_0 < t_1 < \cdots < t_N = T - \delta$ of the reverse CTMC and a sequence of score estimates $\{s_{T - t_k}: k = 0, 1, \cdots, N - 1\}$, where $s_t(y, x)$ estimates $p_t(y) / p_t(x)$, we define the associated aggregated score-entropy loss as follows:
\begin{align}
\label{eq:aggregated-SE}
	\cE_{\rm SE} = \sum_{k = 0}^{N - 1} (t_{k + 1} - t_k) \cdot  \cL_{\SE}(s_{T - t_k}, p_{T - t_k}, T - t_k). 
\end{align} 
In the above display, 
\begin{align}
\label{eq:SE-t}
    \cL_{\SE}(s_t, p_t, t) = \E_{x \sim p_t} \Big[ \sum_{y \in N(x)} \overrightarrow{Q}_t(y, x) \Big( s_t(y, x) - \frac{p_t(y)}{p_t(x)}  - \frac{p_t(y)}{p_t(x)} \log \frac{ s_t(y, x)}{{p_t(y) / p_t(x)}}\, \Big) \Big], 
\end{align}
where for $x \in \mathcal{X}$, we define $N(x) = \{y: \overrightarrow{Q}_t(y, x) > 0\}$.
The definition of $N(x)$ is explicit under both the uniform and absorbing rate matrices. 
In addition, for both the uniform and absorbing rate matrices introduced in Section \ref{sec:forward-reverse-CTMC}, these matrices are time-homogeneous, and thus $N(x)$ does not change with $t$. 
Accordingly, we sometimes write $\overrightarrow{Q}_t = \overrightarrow{Q}$. 

\subsection{Minimax lower and upper bounds for a single score-entropy loss}

Our goal is to characterize the sample complexity of learning concrete score functions under the aggregated score-entropy loss \eqref{eq:aggregated-SE}. To this end, we first analyze the minimax score-entropy loss for a single distribution $p$ over a finite product space $\mathcal{X} = \mathcal{V}^d$, where for uniform discrete diffusion $\mathcal{V} = [S]$ and for masking discrete diffusion $\mathcal{V} = [S] \cup \{\mask\}$. 
We let $v = |\cV|$. 
To exclude degenerate cases, throughout this work we assume $S \geq 2$.

For notational simplicity, we write $p_x = p(x)$. 
We want to estimate the score function associated with $p$ based on $n$ independent and identically distributed (i.i.d.) samples $Z_1^n = (Z_1, Z_2, \dots, Z_n)$ drawn from $p$.
 Specifically, for an estimate $s(y, x)$ of the ratio $p_y/p_x$ constructed from $Z_1^n$, its associated score-entropy loss under the target distribution $p$ is defined as
 \begin{align}
 \label{eq:SE}
 	\cL_{\SE}(s, p) = \E_{x \sim p} \Big[ \sum_{y \in N(x)} \overrightarrow{Q}(y, x) \Big( s(y, x) - \frac{p_y}{p_x}  - \frac{p_y}{p_x} \log \frac{ s(y, x)}{{p_y / p_x}}\, \Big) \Big]. 
 \end{align}
Objective \eqref{eq:SE} suggests that the difficulty of estimation is closely tied to the regularity of the density ratios of adjacent states, motivating us to focus on a family of distributions whose density ratios are uniformly bounded. 
In particular, for $\zeta > 1$, we define
\begin{align*}
	\cP_{\zeta} = \Big\{ p: \,\,  \frac{p(y)}{p(x)} \leq \zeta\,\, \mbox{ for all }x, y \in \mathcal{X} \mbox{ and }y \in N(x) \Big\}. 
\end{align*}
\rev{The bounded neighboring density ratio condition rules out a statistically degenerate regime. When $p_x$ is small, few or no observations may take the value $x$, making quantities associated with $x$ difficult to estimate. This issue becomes particularly severe when $p_x$ is much smaller than a neighboring probability $p_y$, since the target score $p_y/p_x$ can then become arbitrarily large. Thus, an unbounded neighboring density ratio may require estimating an arbitrarily large score from very limited information, leading to a correspondingly large estimation error. The condition $p_y/p_x\leq\zeta$ rules out this regime.}

Given $n$ samples, the minimax risk for estimating the score over the distribution class $\cP_{\zeta}$ under the score-entropy loss \eqref{eq:SE} is defined as 
\begin{align}
\label{eq:minimax}
	\cL_{n}(\zeta) = \inf_{s} \sup_{p \in \cP_{\zeta}} \E \big[ \cL_{\SE}(s, p) \big], 
\end{align}
where the expectation is taken over $Z_1^n \sim_{i.i.d.} p$, and the infimum is over all score estimates $s$ that are measurable with respect to $Z_1^n$.

To establish our main theorem, we impose the following assumptions. 

\begin{ass}
\label{ass:n-large}
We impose the following assumptions: 
\begin{enumerate}
	 \rev{\item We assume that the sample size $n$ is large enough such that 
    \[
    n
    \geq \max \Big\{ \frac{9}{4}|\cX|, \, 6, \, 100 \log (nvd \zeta) \Big\}, \qquad \frac{2}{1.49}vd\, \zeta(\log \zeta + 1)\,e^{-0.4802 n} \leq \frac{|\cX|}{100 n}. 
    \]}
    \item We assume that $ \zeta > 2$. 
\end{enumerate}
	 
\end{ass}

\begin{rem}
	The first condition in Assumption~\ref{ass:n-large} is necessary: without additional structural assumptions on $p$ (e.g., sparsity), the sample size must scale at least with the number of states for the score to be estimable from an information-theoretic perspective.
	The second condition corresponds to considering a moderately large value of $\zeta$. 
	The constants $9 / 4,\, 8$ and $2$ can be replaced by other values. 
\end{rem}

Our first contribution is to establish a lower bound on the minimax risk \eqref{eq:minimax}.

\begin{thm}
\label{thm:minimax-lower-bound}
	Under Assumption~\ref{ass:n-large}, there exists a positive numerical constant $C_0$ such that, for both the uniform and absorbing transition rate matrices, the following holds:
	\begin{align*}
		\cL_{n}(\zeta) \geq \frac{C_0 |\mathcal{X}| }{n}. 
	\end{align*} 
\end{thm}
\begin{proof}[Proof of Theorem \ref{thm:minimax-lower-bound}]
	We prove Theorem \ref{thm:minimax-lower-bound} in Appendix \ref{proof:thm:minimax-lower-bound}. 
\end{proof}


\rev{The minimax lower bound is established by constructing a hard family of distributions and reducing score estimation to a collection of binary testing problems. Specifically, we consider a Poissonized sampling model and construct a product prior where each probability mass $p_x$ independently takes one of two values. Under this prior, estimating each neighboring density ratio $p_y/p_x$ requires distinguishing two nearby Poisson distributions. Standard testing lower bounds imply that the estimation error for each local density ratio is at least of order $1/n$ under the score-entropy loss. Since the score-entropy loss aggregates the estimation error over all neighboring pairs, and the total transition weight is of order $|\mathcal{X}|$, summing over all neighboring pairs yields a minimax risk of order $|\mathcal{X}|/n$.}

The rate derived in Theorem \ref{thm:minimax-lower-bound} is consistent with the usual minimax scaling (viewing the number of states $|\mathcal{X}|$ in the target distribution as the effective number of parameters). 
This problem is quite different from standard parametric estimation and is closely related to discrete distribution estimation in the high-dimensional regime, see \cite{devroye2001combinatorial} for a discussion of developments in this area.

Next, we present an algorithm that achieves the rate in Theorem~\ref{thm:minimax-lower-bound} up to poly-logarithmic factors.
For \(x \in \mathcal{X}\), define \(N_x = \sum_{i=1}^n \mathbf{1}\{Z_i = x\}\).
We consider the following MLE thresholding algorithm: 
\begin{algorithm}
\caption{MLE thresholding}\label{alg:thresholding}
\begin{algorithmic}
\For{$x \in \cX$}
\For{$y \in N(x)$}
\If{$\min \{N_x, N_y\} \geq \log n$}
\State Set $s(y, x) \gets N_y / N_x$
\Else
\State Set $s(y, x) \gets 1 / n$
\EndIf
\EndFor
\EndFor \\
\Return $s$
\end{algorithmic}
\end{algorithm}

\begin{revision}
\begin{rem}
    Algorithm \ref{alg:thresholding} is simple to implement and computationally efficient, whereas the algorithms in \cite{srikanth2026discrete,wakasugi2025state,su2025theoretical,wan2025discrete} assume finding exact solutions of potentially complex and non-convex empirical risk minimization problems. Since the score estimator of \cref{alg:thresholding} is determined by empirical counts, its storage complexity is bounded by the size of the training dataset.
\end{rem}
\end{revision}

We next provide theoretical guarantees for Algorithm~\ref{alg:thresholding}.

\begin{thm}
\label{thm:thresholding}

Denote by $s$ the output of Algorithm \ref{alg:thresholding}, then there exists a positive numerical constant $C_1 > 0$, such that 
\begin{align*}
	\E[\cL_{\SE}(s, p)] \leq \frac{ C_1|\cX| \log_v |\cX| \zeta \log n \log (n \zeta)}{n}.  
\end{align*}
%
%
\end{thm} 
\begin{proof}[Proof of Theorem \ref{thm:thresholding}]
	We prove Theorem \ref{thm:thresholding} in Appendix \ref{proof:thm:thresholding}. 
\end{proof}

\subsection{Implications on the aggregated score-entropy loss}
\label{sec:aggregated-score}

When $\zeta$ is treated as a constant, the minimax lower and upper bounds established in Theorems~\ref{thm:minimax-lower-bound} and \ref{thm:thresholding} differ only by polylogarithmic factors. However, it is unclear in practice whether $\zeta$ can indeed be regarded as a constant.
In particular, a larger $\zeta$ results in a wider gap between the lower and upper bounds. 

In the next part, we show that this issue is largely mitigated when considering the aggregated score-entropy loss. 
For the intermediate distributions induced by both uniform and masking discrete diffusion models, the density ratios are automatically controlled as a function of time. 
This yields upper and lower bounds for the minimax aggregated score-entropy error that differ by at most poly-logarithmic factors.

To introduce our results, for $0 \leq t \leq T$, we define the time-dependent density ratio upper bound
\begin{align*}
	\zeta_t = \sup_{x \in \cX, y \in N(x)} \frac{p_t(y)}{p_t(x)}, 
\end{align*}
where we recall $p_t$ represents the marginal distribution in the forward process, and $p_0$ denotes the target distribution. 
The following two lemmas provide assumption-free bounds on $\zeta_t$, under both uniform and absorbing rate matrices.

\begin{lem}
\label{lem:31}
	Consider the uniform rate matrix \eqref{eq:uniform}, then for any $p_0$, 
	\begin{align*}
		\zeta_t \leq \frac{1 + (v - 1)e^{-t}}{1 - e^{-t}}. 
	\end{align*}
\end{lem}

\begin{proof}[Proof of Lemma \ref{lem:31}]
We prove Lemma \ref{lem:31} in Appendix \ref{proof:lem:31}. 
\end{proof}

\begin{lem}
\label{lem:32}
	Consider the absorbing rate matrix \eqref{eq:absorbing}, then for any $p_0$,  
	\begin{align*}
		\zeta_t \leq \frac{e^{-t}}{1 - e^{-t}}. 
	\end{align*}
\end{lem}

\begin{proof}[Proof of Lemma \ref{lem:32}]
    We prove Lemma \ref{lem:32} in Appendix \ref{proof:lem:32}. 
\end{proof}

Combining the above lemmas with Theorem \ref{thm:thresholding}, we arrive at the following conclusion: 
\begin{thm}
\label{thm:aggregated}
	Consider a discretization scheme $0 = t_0 < t_1 < \cdots < t_N = T - \delta$. For $k \in \{0, 1, \cdots, N - 1\}$, denote by $s_{T - t_k}$ the output of Algorithm \ref{alg:thresholding} constructed using $n$ samples from $p_{T- t_k}$, then there exists a positive numerical constant $C_2$, such that 
	\begin{align*}
		& \sum_{k = 0}^{N - 1} (t_{k + 1} - t_k) \cdot \E[ \cL_{\SE}(s_{T - t_k}, p_{T - t_k}, T - t_k) ] \\
        & \leq \frac{ C_2\, v |\cX| \log_v |\cX| (\log n)^2}{n} \cdot \Big(T + \log\frac{1}{\delta}\Big) \cdot  \Big(1 + \log v + \log\Big(1 + \frac{1}{\delta}\Big)\Big). 
	\end{align*} 
\end{thm}
\begin{proof}[Proof of Theorem \ref{thm:aggregated}]
	By Lemma \ref{lem:31} and \ref{lem:32}, we see that for both the uniform and absorbing rate matrices, it holds that $\zeta_t \leq v / (1- e^{-t})$. 
	Therefore, 
	\begin{align*}
		& \sum_{k = 0}^{N - 1} (t_{k + 1} - t_k) \cdot \E[ \cL_{\SE}(s_{T - t_k}, p_{T - t_k}, T - t_k)] \\
		& \leq \sum_{k = 0}^{N - 1} (t_{k + 1} - t_k) \cdot \frac{ C_1|\cX| \log_v |\cX| \zeta_{T - t_k} \log n \log (n \zeta_{T - t_k})}{n} \\
		& \lesssim  \frac{ |\cX| \log_v |\cX| (\log n)^2}{n} \sum_{k = 0}^{N - 1} (t_{k + 1} - t_k) \cdot \frac{v}{1 - e^{-T + t_k}} \cdot \Big(1 + \log\frac{v}{1 - e^{-T + t_k}} \Big)\\
        & \leq \frac{v |\cX| \log_v |\cX| (\log n)^2}{n}  \cdot \int_{\delta}^T \frac{1}{1 - e^{-x}} \dd x  \cdot \Big(1 + \log\frac{v}{1 - e^{-\delta}} \Big) \\
		& \leq  \frac{v |\cX| \log_v |\cX| (\log n)^2}{n} \cdot \Big(T + \log\frac{1}{\delta}\Big) \cdot  \Big(1 + \log v + \log\Big(1 + \frac{1}{\delta}\Big)\Big), 
	\end{align*}
	where ``$\lesssim$'' hides numerical constants. The proof is done. 
\end{proof}

Note that to obtain $n$ samples from $p_{T - t_k}$, it suffices to draw $n$ samples from $p_0$ and apply the forward noising process. 
Therefore, Theorem~\ref{thm:aggregated} yields an upper bound on the aggregated score estimation error based on $n$ samples from $p_0$. 

Recent developments in discrete distribution estimation suggest that the minimax rate for estimation under KL divergence scales as $\tilde{O}(|\mathcal{X}| / n)$, where $\tilde{O}(\cdot)$ hides poly-logarithmic factors  \cite{van2025nearly}. 
\rev{Combining the minimax lower bound in \cite{van2025nearly} with the KL error decomposition \eqref{eq:three-terms} for $\tau$-leaping and uniformization, and taking the time horizon $T$ sufficiently large and the discretization sufficiently fine so that the initialization and discretization errors are negligible, we obtain a minimax lower bound of $\widetilde{\Omega}(|\mathcal{X}|/n)$ for the aggregated score estimation error. Hence, the upper bound established in Theorem~\ref{thm:aggregated} is nearly tight up to polylogarithmic factors.}
Consequently, Theorem~\ref{thm:aggregated} suggests that with sufficiently accurate initialization and fine discretization, the resulting discrete diffusion model as a distribution estimator can achieve the minimax KL estimation rate (up to poly-logarithmic factors). 

\rev{
Theorem~\ref{thm:aggregated} shows that the contribution of the score estimation error, $\cE_{\rm SE}$, to the KL guarantee scales as $\widetilde{O}(|\mathcal{X}|/n)$. Substituting this bound into Eq.~\eqref{eq:three-terms} and choosing the time horizon and discretization scheme appropriately so that $\cE_{\rm init}$ and $\cE_{\rm dis}$ are no larger than the score estimation error term yields
\[
    \KL(p_{\rm target}\parallel p_{\rm output})
    = \widetilde{O}\!\left(\frac{|\mathcal{X}|}{n}\right).
\]
Thus, the near-minimax score estimation rate established in this paper translates directly into a near-minimax KL distribution estimation rate for the final generated distribution.}

\begin{revision}
\subsection{Comparison with recent work on concrete score estimation}

In this section, we compare our results with recent work on concrete score estimation. 

\paragraph{Concrete score estimation.}
\cite{srikanth2026discrete} studies the sample complexity of concrete score estimation using a neural-network-based score estimator trained by minimizing the empirical score-entropy loss. They establish upper bounds on the KL divergence between the target and output distributions that scale as $\widetilde{O}(n^{-1/2})$ in the number of training samples, equivalently requiring $\widetilde{O}(\epsilon^{-2})$ samples to achieve KL error $\epsilon$. Their analysis relies on several assumptions, including smoothness of the population loss, the Polyak--{\L}ojasiewicz condition, and uniform upper and lower bounds on the true score.
\cite{wakasugi2025state} also studies a neural-network-based score estimator trained by minimizing the empirical score-entropy loss. They obtain an upper bound on the KL divergence that scales as $\widetilde{O}(n^{-1})$ with the number of training samples. Their analysis assumes that both the true and learned score functions are uniformly bounded above and below and that the modified log-Sobolev constant is bounded away from zero.

\paragraph{Discrete flow matching.}
Discrete flow matching \cite{gat2024discrete} generalizes discrete diffusion models by allowing more general probability paths, with concrete score estimation replaced by learning the associated velocity fields. 
Among works studying the statistical complexity of discrete flow matching, \cite{su2025theoretical} proposes to embed the discrete state space into a continuous space.
Under the assumption that the true velocity field belongs to a Hölder class, they analyze a factorized velocity-field estimator parameterized by a transformer. 
Their resulting velocity-field estimation error scales as
\[
    \widetilde{O}\!\left(
        S^{13d_0} n^{-1/(5Sd_0)}
    \right),
\]
where $S$ is the vocabulary size and $d_0$ is the transformer feature dimension. Combining this estimation result with their analysis of the discrete flow yields the following total variation bound:
\[
    \widetilde{O}\!\left(
        S^{7d_0} n^{-1/(9Sd_0)}
    \right).
\]
\cite{wan2025error} also studies finite-sample velocity estimation for discrete flow matching. They analyze a neural-network-based estimator under the assumption that the true velocity field belongs to a Hölder class with smoothness parameter $\beta$. Their resulting total variation bound takes the form
\[
    \widetilde{O}\!\left(
        S^{\beta+1}d^{\lfloor\beta\rfloor+1}\tau^{-1/2}
        (n\tau^2)^{-\frac{\beta}{2(2\beta+d+1)}}
        + \tau d
    \right),
\]
where $d$ denotes the data dimension and $\tau$ is the early-stopping parameter.

\paragraph{Comparison with our results.}
Our results complement these recent analyses by addressing a different statistical question. 
We establish a minimax lower bound for concrete-score estimation under the score-entropy loss and construct an MLE-based thresholding estimator that matches this lower bound up to polylogarithmic factors. Specifically, the resulting minimax score-entropy risk scales as
\[
    \widetilde{\Theta}\!\left(\frac{|\mathcal{X}|}{n}\right).
\]
Combined with the KL decomposition in Eq.~\eqref{eq:three-terms}, this yields a KL distribution estimation error of order $\widetilde{O}(|\mathcal{X}|/n)$. 
For fixed problem parameters, this corresponds to an $n^{-1}$ dependence, which is faster than the $n^{-1/2}$ dependence in \cite{srikanth2026discrete} and the different powers of $n$ arising in the discrete flow-matching results discussed above. Notably, \cite{wakasugi2025state} also obtains an $n^{-1}$ dependence, while our result further provides a matching minimax lower bound and applies under fewer structural assumptions. In particular, \cite{wakasugi2025state} imposes smoothness assumptions that allow information to be shared across the domain, whereas our analysis does not require such structural assumptions. 
Indeed, for the final distribution estimation guarantee implied by Theorem~\ref{thm:aggregated}, we impose no additional structural assumptions on the target distribution. Consequently, our rate retains dependence on the full state space size $|\mathcal{X}|$. Importantly, our minimax lower bound shows that this linear dependence on $|\mathcal{X}|$ is unavoidable in the absence of additional structural assumptions. The related works discussed above consider different, more structured statistical settings and therefore exhibit different dependencies on the dimension, vocabulary size, and state space size.
\end{revision}

\section{Numerical experiments}
\label{sec:experiments}

In this section, we present numerical experiments to support our findings. 
For this experiment, we set the sample space as $\{0, 1\}^{10}$, and define the target distribution $p_0$ as the uniform distribution over all states in $\mathcal{X}$ whose entries sum to 2: 
\begin{align}
\label{eq:p0}
	p_0 = \mbox{Unif}\Big(\big\{ x \in\{0, 1\}^{10}: \sum_{i = 1}^{10} x_i = 2 \big\} \Big). 
\end{align} 
With $n$ samples i.i.d. from $p_0$, our goal is to construct a discrete diffusion model that generates new samples from $p_0$. 
For score estimation, we use the MLE thresholding algorithm described in Algorithm~\ref{alg:thresholding}.
For sampling, we adopt the $\tau$-leaping algorithm \cite{gillespie2001approximate}. 

\subsection{KL sampling error scales linearly with the aggregated score}
 
Recall that Eq.~\eqref{eq:three-terms} suggests that the KL divergence between the target and output distributions can be controlled linearly by the aggregated score estimation error. 
In the first experiment, we verify that this upper bound is tight by presenting a concrete example in which the KL divergence empirically exhibits a linear dependence on the aggregated score.

In this experiment, we focus on the target distribution $p_0$ from Eq.~\eqref{eq:p0}. 
We set $\delta = 0.02$, $T = 1.0$, and employ a discretization scheme with step size $0.02$. 
We present our numerical results in Figure \ref{fig:linear}. 
In this figure, the $x$-axis represents the aggregated score estimation error $\mathcal{L}_{\mathrm{SE}}$, and the $y$-axis shows the KL divergence between the target and output distributions. 
The numbers in the plot indicate the number of training samples used. 
The left panel corresponds to uniform discrete diffusion, and the right panel to masked diffusion. 
Both panels exhibit a clear linear trend, indicating that the sampling error scales linearly with the aggregated score error and that this dependence cannot be improved. 
Note that uniform discrete diffusion achieves significantly smaller error than masking discrete diffusion. This is because in the uniform setting $\mathcal{X} = \{0, 1\}^{10}$, while in the masking setting $\mathcal{X} = \{0, 1, \mask\}^{10}$, which has a substantially larger sample space. 
\begin{figure}[ht!]
\centering
\begin{subfigure}{.5\textwidth}
  \centering
  \includegraphics[width=\linewidth]{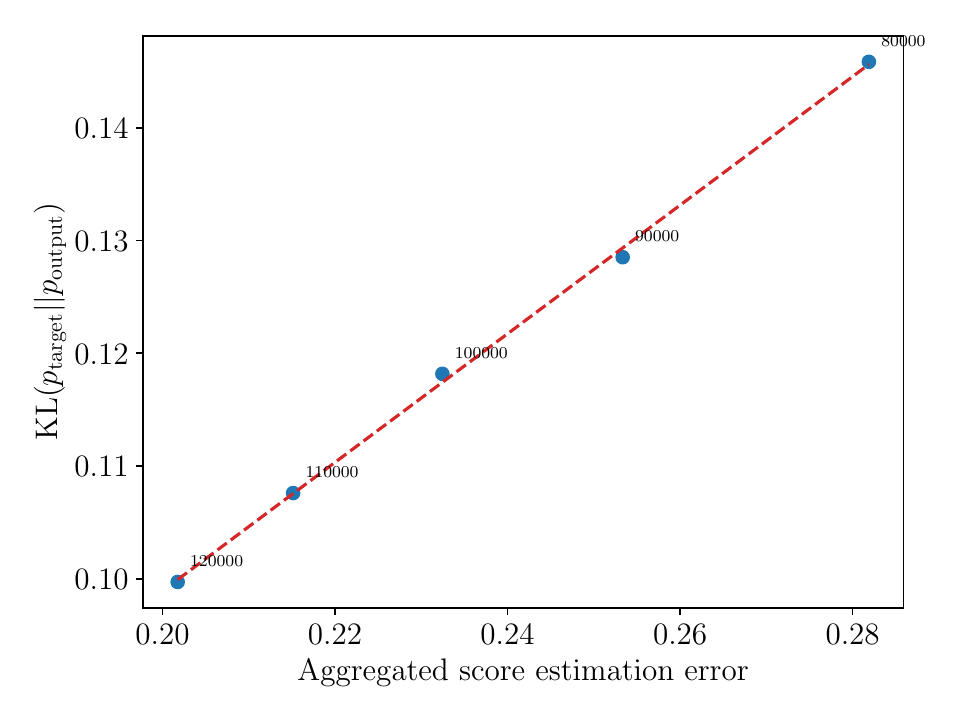}
\end{subfigure}%
\begin{subfigure}{.5\textwidth}
  \centering
  \includegraphics[width=\linewidth]{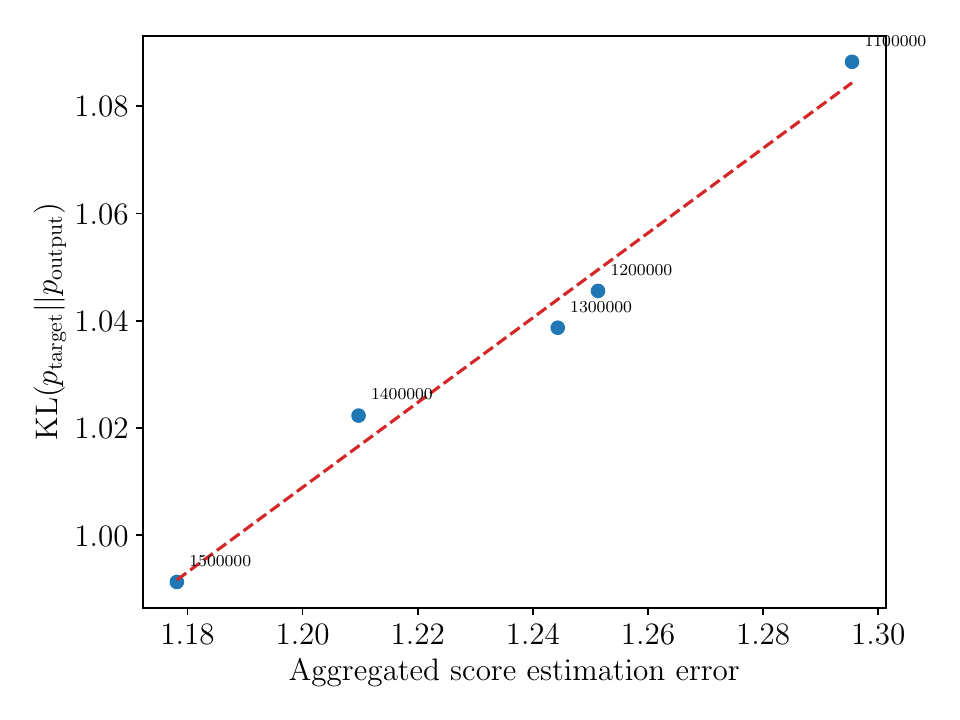}
\end{subfigure}
\caption{KL divergence between the target and output distributions as a function of the aggregated score estimation error. The left panel presents results for uniform discrete diffusion, and the right panel shows results for masking discrete diffusion.}
\label{fig:linear}
\end{figure} 

\subsection{MLE thresholding achieves optimal sample complexity}

In the next experiment, we show that MLE thresholding achieves an aggregated score estimation error that scales inversely with the sample size, confirming the theoretical findings in Section~\ref{sec:aggregated-score}. 
For this experiment, we set $p_0$ as in \cref{eq:p0}, $T = 1.0$, $\delta = 0.02$, and use a discretization scheme with step size $0.02$. 
We present the experimental results in Figure~\ref{fig:inv}. 
The $y$-axis reports the aggregated score estimation error, which is observed to scale approximately inversely with the sample size.

\begin{figure}[ht!]
\centering
\begin{subfigure}{.5\textwidth}
  \centering
  \includegraphics[width=\linewidth]{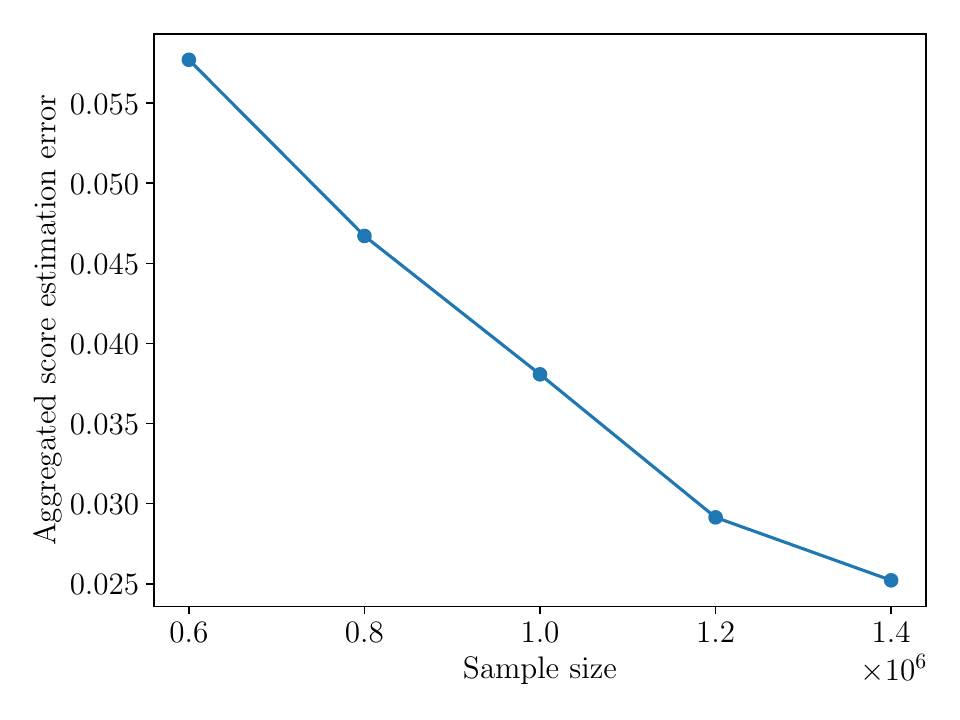}
\end{subfigure}%
\begin{subfigure}{.5\textwidth}
  \centering
  \includegraphics[width=\linewidth]{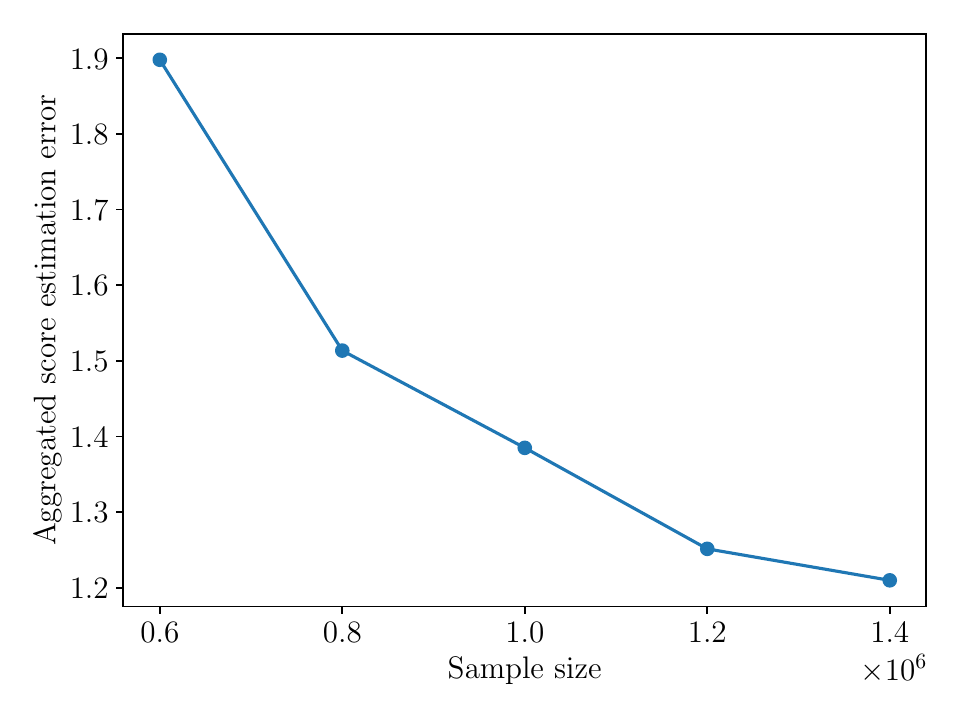}
\end{subfigure}
\caption{Aggregated score estimation error as a function of the sample size. The left panel presents results for uniform discrete diffusion, and the right panel shows results for masking discrete diffusion.}
\label{fig:inv}
\end{figure} 

\begin{revision}

\subsection{Dependence on the state space size}

Finally, we fix the number of training samples at $n=20{,}000{,}000$ and vary
\[
    d\in\{5,9,13,16,20\},
    \qquad
    |\cX|=2^d.
\]
Figure~\ref{fig:state-space} shows that the score estimation problem becomes increasingly difficult as the state space size grows. In particular, the aggregated score estimation error increases from approximately $5.79\times10^{-6}$ at $|\cX|=2^5$ to $0.793$ at $|\cX|=2^{20}$, while the endpoint KL divergence increases from approximately $0.052$ to $1.080$.
These numerical results are consistent with the derived minimax error rate's dependence on the state space size. 

\begin{figure}[ht!]
\centering
\begin{subfigure}[t]{.49\textwidth}
  \centering
  \includegraphics[width=\linewidth]
  {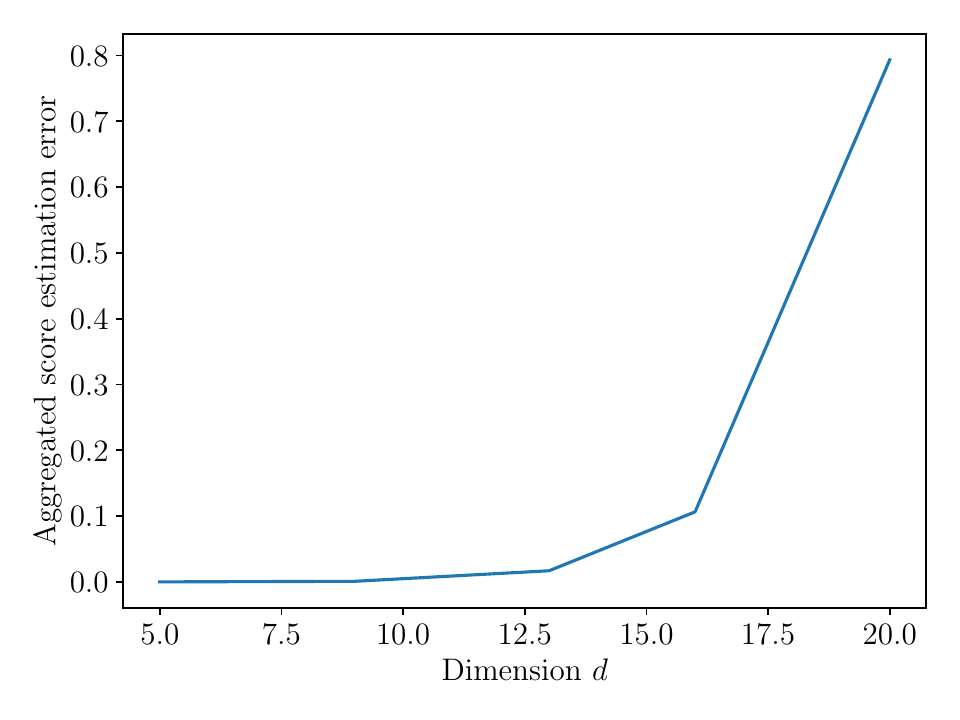}
  \caption{Aggregated score-estimation error.}
\end{subfigure}
\hfill
\begin{subfigure}[t]{.49\textwidth}
  \centering
  \includegraphics[width=\linewidth]
  {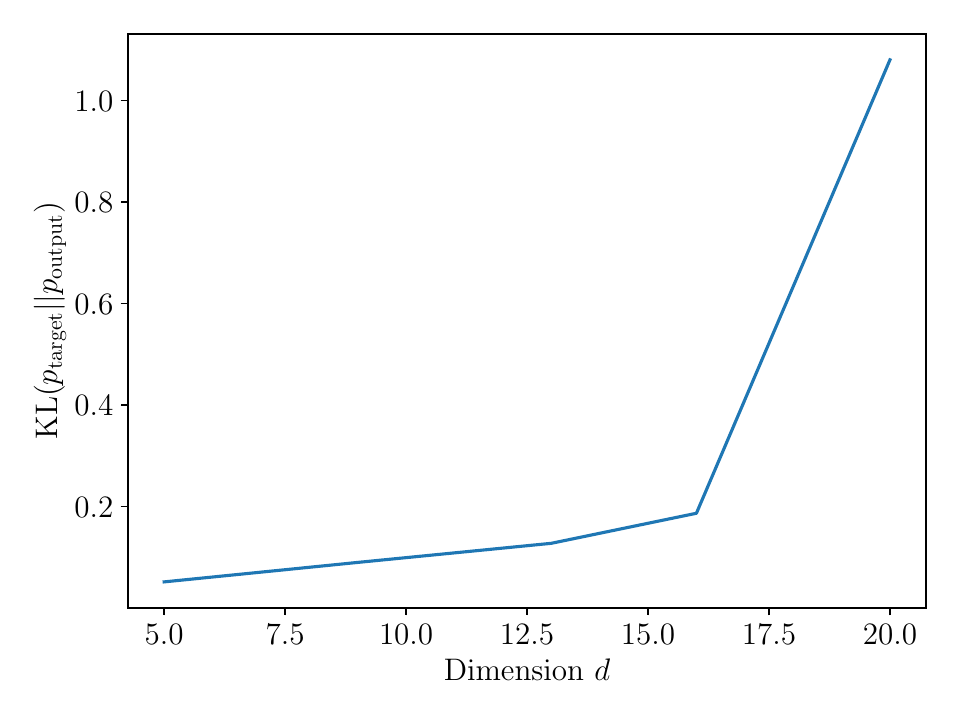}
  \caption{Endpoint KL divergence.}
\end{subfigure}
\caption{Dependence on the state-space size at the fixed sample size
$n=20{,}000{,}000$. Both score-estimation and sampling errors increase
as $|\cX|=2^d$ grows.}
\label{fig:state-space}
\end{figure}

\end{revision}

\section{Discussion}
\label{sec:discussion}

In this work, we study the sample complexity of estimating concrete score functions for discrete diffusion models under the score-entropy loss.
\rev{We establish two minimax results. First, Theorems~\ref{thm:minimax-lower-bound} and~\ref{thm:thresholding} characterize the minimax error rate of single score estimation under the score-entropy loss for a class of distributions with bounded neighboring density ratios. 
Specifically, we derive a minimax lower bound, and propose an MLE-based thresholding estimator that matches this lower bound up to constant and polylogarithmic factors. }
Our minimax lower and upper bounds differ by a factor that depends on the neighboring density ratio, which is not directly controllable for general real-world distributions.
On the positive side, we show that for both uniform and masking discrete diffusions, these neighboring density ratios can be automatically controlled, leading to nearly matching lower and upper bounds on the minimax aggregated score estimation error (Theorem~\ref{thm:aggregated}).
This result suggests that with sufficiently accurate initialization and discretization, SEDD achieves optimal minimax sample complexity, as measured by the KL divergence between the target and output distributions, providing our second minimax result.
Our analysis is closely connected to high-dimensional discrete distribution estimation and can be viewed as an application of this line of work to the diffusion model setting.

Several interesting directions remain open. For instance, our minimax lower and upper bounds for single score estimation differ by a factor depending on the upper bound of the neighboring density ratio, and it remains unclear whether this dependency can be removed.
Moreover, our current analysis does not impose any structural assumptions on the target distribution, so the resulting minimax rate scales with the ambient dimension. In practice, however, target distributions often exhibit low-dimensional structure and possess a much smaller intrinsic dimension. An important future direction is to extend our analysis to such settings and establish minimax rates that depend on the intrinsic dimension rather than the ambient dimension.

\bibliographystyle{alpha}
\bibliography{bib}

\newpage 

\appendix 

\section{Proof of theoretical results}

\subsection{Proof of Theorem \ref{thm:minimax-lower-bound}}
\label{proof:thm:minimax-lower-bound}

Throughout the proof, we treat probability distributions over $\mathcal{X}$ as $|\mathcal{X}|$-dimensional vectors with non-negative entries summing to one. 

\subsubsection*{Step I: Poissonized model and product prior}

A common approach to establishing a lower bound for the minimax risk is to consider the Bayes risk corresponding to a carefully chosen prior. 
In particular, we aim to identify an unfavorable prior $\pi$ and compute its associated Bayes risk.
For computational convenience, we adopt a product prior over $p$, and later use concentration of measure arguments to show that under this prior, $p$ is with high probability close to a proper distribution in $\cP_{\zeta}$. 
We consider a Poissonized model, under which we observe independent counts $N_x \sim \Poi(n p_x)$ for $x \in \mathcal{X}$. 
Given $N = (N_x)_{x \in \mathcal{X}}$, we want to estimate the density ratios $r_{y, x} = p_y / p_x$ for $x, y \in \mathcal{X}$ and $y \in N(x)$.

Under this Poissonized model, for each pair $x, y \in \mathcal{X}$ with $y \in N(x)$, the counts $(N_x, N_y)$ form a sufficient statistic for $r_{y, x}$. 
Hence, it suffices to consider estimators based solely on $(N_x,N_y)$ of the form: 
\begin{equation}
\label{eq:pairwise-est}
\hat r_{y, x} = g_{y, x}(N_y,N_x). 
\end{equation} 
For a prior $\pi$, we can then define its Bayes risk under such Poissonized model: 
\begin{align}
	& \cR_B^{\pos}(\pi) = \inf_{g}\ \E_{p \sim \pi}\Big[\,\E_{N}\Big[\sum_{x \in \mathcal{X}} p_x\,  \sum_{y\in N(x)} \overrightarrow{Q}(y, x)\, D_0\big(r_{y, x},\,g_{y, x}(N_y,N_x) \big) \, \big| \, p \Big]\,\Big], \label{eq:poisson-bayes} \\
	& D_0(r, \hat r\,) = \hat r - r - r \log ({\hat r} / {r}). \label{eq:D0}
\end{align}
In \cref{eq:poisson-bayes}, the outer expectation is taken with respect to the prior $\pi$, and the inner expectation is over independent Poisson counts $N_x \sim \mathrm{Poisson}(n p_x)$. 
We consider a product prior $\pi$, where under this prior  
\begin{align*}
	p_x \sim_{i.i.d.} \mbox{Unif} \Big\{\,\frac{1 + \eta}{|\cX|},\, \frac{1 - \eta}{|\cX|} \Big\}. 
\end{align*}
In the above equation, $\eta > 0$ is a positive parameter that satisfies 
\begin{align*}
	0 < \eta < \frac{\zeta - 1}{\zeta + 1}. 
\end{align*}
%
%
%
%
Since the Bayes risk in \eqref{eq:poisson-bayes} is a double sum and the estimators $g_{y, x}$ are specified for each pair $(y, x)$, we conclude that the Bayes risk admits the following decomposition:
\begin{align*}
& \cR_B^{\pos}(\pi)\ =\ \sum_{x\in \mathcal{X}}\sum_{y\in N(x)} \cR_B^{\pos}(y, x; \pi),\\
&\cR_B^{\pos}(y, x; \pi) =\inf_{g_{y, x}}\ \E_{p \sim \pi}\Big[\,\E_{N}\Big[\, p_x\,   \overrightarrow{Q}(y, x)\, D_0\big(r_{y, x},\,g_{y, x}(N_y,N_x) \big) \, \big| \, p \Big]\,\Big].
\end{align*}
%
We then lower bound each of the pairwise Bayes risk $\cR_B^{\pos}(y, x; \pi)$, where $x, y \in \cX$.
Conditional on $p$, for any $x, y \in \cX$ with $y \in N(x)$, the pair $(N_x,N_y)$ follows a product distribution $\Poi(n p_x)\otimes \Poi(n p_y)$, so
\begin{equation}
\label{eq:expand-sum}
\E_N\big[\,D_0\big(r_{y, x},\,g_{y, x}(N_y,N_x) \big) \, \big| \, p\,\big]=\sum_{j,k\ge 0} D_0\big(r_{y, x},\, g_{y, x}(k, j)\big)\, q^\Poi_{n p_x}(j)\, q^\Poi_{n p_y}(k)~, 
\end{equation}
where $q^{\Poi}_{\lambda}$ denotes the Poisson density with mean $\lambda$. 
Let $\lambda_\pm=n(1\pm\eta) / |\cX|$. Averaging \eqref{eq:expand-sum} over $(p_x,p_y)\in\{(1\pm\eta)/|\cX|\}^2$, we obtain that
\begin{align}
& \cR_B^{\pos}(y, x; \pi) \nonumber \\
&= \frac{\overrightarrow{Q}(y, x)}{4} \inf_{g_{y, x}} \sum_{j,k\ge 0}\Big[ \frac{1 + \eta}{|\cX|} \underbrace{q^\Poi_{\lambda_+}(j) q^\Poi_{\lambda_+}(k)}_{W^{++}_{jk}}\,D_0(1,\,g_{y, x}(k, j)) +\frac{1 - \eta}{|\cX|}\underbrace{q^\Poi_{\lambda_-}(j) q^\Poi_{\lambda_-}(k)}_{W^{--}_{jk}}\,D_0(1,\,g_{y, x}(k, j)) \nonumber\\
&+\frac{1 + \eta}{|\cX|}\underbrace{q^\Poi_{\lambda_+}(j) q^\Poi_{\lambda_-}(k)}_{U_{jk}}\, D_0\big(\tfrac{1-\eta}{1+\eta},\,g_{y, x}(k, j)\big) 
+\frac{1 - \eta}{|\cX|}\underbrace{q^\Poi_{\lambda_-}(j) q^\Poi_{\lambda_+}(k)}_{V_{jk}}\, D_0\big(\tfrac{1+\eta}{1-\eta},\,g_{y, x}(k, j)\big)\Big]. \label{eq:four-terms}
\end{align}
Note that for $a,b,u,v\ge0$, $(ua+vb)/2\ge \min\{u,v\}\,(a+b)/2$. 
Apply this to both the opposite-sign pair and the same-sign pair in \eqref{eq:four-terms}, we get 
\begin{align*}
\label{eq:min-extract}
\begin{split}
& \cR_B^{\pos}(y, x; \pi) \\
\ &\ge\ \frac{\overrightarrow{Q}(y, x)}{4|\cX| }\, \inf_{g_{y, x}}\sum_{j,k\ge0}\Big[\min\{U_{jk},V_{jk}\}\,\Big( (1 + \eta) D_0(\tfrac{1-\eta}{1+\eta},g_{y, x}(k, j))+ (1 - \eta)D_0(\tfrac{1+\eta}{1-\eta},g_{y, x}(k, j))\Big)\\
&\hspace{7.5em}+2\min\{W^{++}_{jk},W^{--}_{jk}\}\,D_0(1,g_{y, x}(k, j))\Big].
\end{split}
\end{align*}
Since $\inf_{a > 0} D_0(1,a)=0$ (attained at $a=1$), therefore: 
\begin{equation}
\label{eq:drop-samesign}
\cR_B^{\pos}(y, x; \pi) \ge \frac{\overrightarrow{Q}(y, x)}{4|\cX|}\sum_{j,k\ge0}\min\{U_{jk},V_{jk}\}\cdot {\inf_{a>0}\big[ (1 + \eta) D_0(\tfrac{1-\eta}{1+\eta},a)+ (1 - \eta) D_0(\tfrac{1+\eta}{1-\eta},a)\big]}.
\end{equation}
For $a,b,c,d \geq 0$, it holds that $\min\{ab,cd\}\ge \min\{a,c\}\min\{b,d\}$. Hence, 
\begin{align*}
\label{eq:overlap}
\begin{split}
\sum_{j,k \geq 0}\min\{U_{jk},V_{jk}\}
&\ge\Big(\sum_{j \geq 0}\min\{q^\Poi_{\lambda_+}(j),\,q^\Poi_{\lambda_-}(j)\}\Big)
 \Big(\sum_{k \geq 0}\min\{q^\Poi_{\lambda_+}(k),\,q^\Poi_{\lambda_-}(k)\}\Big)\\
&=\big(1-\TV(\Poi(\lambda_-),\Poi(\lambda_+))\big)^2~.
\end{split}
\end{align*}
We next analyze $\inf_{a>0}\big[ (1 + \eta) D_0(\tfrac{1-\eta}{1+\eta},a)+ (1 - \eta) D_0(\tfrac{1+\eta}{1-\eta},a)\big]$ in \cref{eq:drop-samesign}. Observe that 
\begin{align*}
	(1 + \eta) D_0(\tfrac{1-\eta}{1+\eta},a)+ (1 - \eta) D_0(\tfrac{1+\eta}{1-\eta},a) = 2( a - 1 - \log a ) + 2\eta \log \Big( \frac{1 + \eta}{1 - \eta} \Big). 
\end{align*}
As a consequence, $\inf_{a>0}\big[ (1 + \eta) D_0(\tfrac{1-\eta}{1+\eta},a)+ (1 - \eta) D_0(\tfrac{1+\eta}{1-\eta},a)\big] = 2\eta \log \frac{1 + \eta}{1 - \eta}$.  
Putting together the above analysis, we conclude that for $x, y \in \cX$ with $y \in N(x)$, 
\begin{equation}%
\label{eq:cRBpos}
	\cR_B^{\pos}(y, x; \pi) \ge \frac{\overrightarrow{Q}(y, x)}{2|\cX| }\cdot \big(1-\TV(\Poi(\lambda_-),\Poi(\lambda_+))\big)^2\cdot \eta \log \frac{1 + \eta}{1 - \eta}. 
\end{equation}
By Eq.~(2.2) of \cite{adell2006exact}, for $t,x\ge0$ we have 
\begin{equation*}
\TV(\Poi(t),\Poi(t+x))\ \le\ \min\Big\{1-e^{-x},\ \sqrt{\tfrac{2}{e}}\,(\sqrt{t+x}-\sqrt{t})\Big\}.
\end{equation*}
Setting $t=\lambda_-$ and $x=\lambda_+-\lambda_-$, we obtain that 
\begin{equation}
\label{eq:TV-bound}
\TV(\Poi(\lambda_-),\Poi(\lambda_+))\ \le\ \min\Big\{1-e^{-2\eta n / |\cX| },\ 2\eta\sqrt{\tfrac{n}{e\,|\cX|}}\Big\}.
\end{equation}
Substituting \cref{eq:TV-bound} back into \cref{eq:cRBpos}, we conclude that for all $x, y \in \cX$ and $y \in N(x)$, 
\begin{align}
\label{eq:pairwise-bayes-risk}
	\cR_B^{\pos}(y, x; \pi) \ge \frac{\overrightarrow{Q}(y, x)}{2|\cX| } \cdot  \max\Big\{e^{-2\eta n / |\cX| },\ 1 - 2\eta\sqrt{\tfrac{n}{e \,|\cX|}}\Big\}^2 \cdot \eta \log \frac{1 + \eta}{1 - \eta}. 
\end{align}
Setting 
\begin{align*}
	\eta = \frac{1}{2} \sqrt{\frac{|\cX|}{n}} \leq \frac{\zeta - 1}{\zeta + 1}.  
\end{align*}
With the above choice of $\eta$, we conclude that 
\begin{align}
\label{eq:18}
	\cR_B^{\pos}(y, x; \pi) \ge \frac{(1 - e^{-1/2})^2\,\overrightarrow{Q}(y, x) \eta^2}{|\cX| }  = \frac{(1 - e^{-1/2})^2\,\overrightarrow{Q}(y, x) }{4n }. 
\end{align}
Now we go back to the full Bayes risk. 
Summing up terms in \cref{eq:18} gives 
\begin{align}
\label{eq:step1-result}
	\cR_B^{\pos}(\pi)\ \geq\ \sum_{x\in \mathcal{X}}\sum_{y\in N(x)} \frac{ \overrightarrow{Q}                                                                                                                                                                                                                                                                                                                                                                                                                                                                                                                                                                                                                                                                   (y, x)}{30n}  \geq \frac{|\mathcal{X}|}{60n}. 
\end{align}

\subsubsection*{Step II: Connecting the Poissonized model to the original model. }

We now connect the Poissonized model to the original model. 
Recall that under the original model,  we observe $Z_1,Z_2, \cdots,Z_n \sim_{i.i.d.}\mathrm{Multi}(p)$ on a finite state space $\cX$ with probability mass functions $p=(p_x)_{x\in \cX}$.
For $x \in \cX$, with a slight abuse of notations, we write $N_x=\sum_{i = 1}^n \1\{Z_i=x\}$.
For every pair $(x, y) \in \cX^2$, note that $(N_x, N_y)$ is also a sufficient statistic for $(p_x, p_y)$ under the original model. 
Therefore, we shall restrict to estimators that take the form of $g_{y, x}(N_y, N_x)$ to estimate $r_{y, x} = p_y / p_x$. 

Given $p = (p_x)_{x \in \cX}$ (not necessarily a probability distribution vector), an estimator $g$, and data $N = (N_x)_{x \in \cX}$, we define 
\[
L(p;g,N)
=\sum_{x \in \cX} p_x  \sum_{y\in N(x)} \overrightarrow{Q}(y, x) D_0\big(r_{y, x},\,g_{y, x}(N_y,N_x)\big).
\]
We note the following scale homogeneity: if $s=\sum_{x \in \cX} p_x$ and $\bar p=p/s$, then
\begin{equation}\label{eq:scale}
L(p;g,N)=s\,L(\bar p;g,N)
\qquad\text{and}\qquad
\frac{p_y}{p_x}=\frac{\bar p_y}{\bar p_x}.
\end{equation}
For $\epsilon\in(0,1 / 2)$, we define 
\begin{align*}
	& \mathcal M_{\epsilon, \zeta}=\Big\{p:\,p_x\ge 0,\ |\sum_{x \in \cX} p_x-1|<\epsilon, \, \frac{p_y}{p_x} \leq \zeta\, \mbox{ for }x, y \in \cX \mbox{ with } y \in N(x)  \Big\}, \\
	& \mathcal M_{\zeta} = \Big\{p: \,p_x \geq 0,\, \sum_{x \in \cX} p_i = 1, \, \frac{p_y}{p_x} \leq \zeta\, \mbox{ for }x, y \in \cX \mbox{ with } y \in N(x)\Big\}. 
\end{align*}
We then define the multinomial and Poisson risks: 
\begin{align*}
	& \cR^{\multi}(n, \zeta) =\inf_{g}\sup_{p\in \mathcal M_{\zeta}}
\E_{\mathrm{Mult}(n,p)}\!\left[\,L(p;g,N)\,\right], \\
& \cR^{\pos}(n,\epsilon, \zeta) =\inf_{g}\sup_{p\in\mathcal M_{\epsilon, \zeta}}
\E_{\mathrm{Pois}(np)}\!\left[\,L(p;g,N)\,\right], 
\end{align*}
where the expectation is over the data $N$.
Fix $\delta>0$ and $m \in \N_+$, there exists a (nearly minimax) estimator $g_{\delta, m}$ such that
\begin{equation}
\label{eq:123B}
\sup_{p\in\mathcal M_{\zeta}}
\E_{\mathrm{Mult}(m,p)}\!\left[L(p;g_{\delta, m},N)\right]
\le \cR^{\multi}(m, \zeta)+\delta\quad\forall m\in\mathbb N_+.
\end{equation}
Recall that under the Poissonized model we observe independent counts $N_x \sim \Poi(n p_x)$ for $x \in \cX$. 
Let $n' = \sum_{x \in \cX} N_x$, 
and define the Poissonized estimator by plugging in the random total count: $g_{\delta}^P=g_{\delta, n'}$. 
Evaluating $\cR^{\pos}$ at $g_{\delta}^P$ gives
\begin{equation}\label{eq:124B}
\cR^{\pos}(n,\epsilon, \zeta)
\le
\sup_{p\in\mathcal M_{\epsilon, \zeta}}
\E_{\mathrm{Pois}(np)}\!\left[\,L(p;g_\delta^P,N)\,\right].
\end{equation}
Under $\mathrm{Pois}(np)$, $n' = \sum_{x \in \cX} N_x\sim\mathrm{Pois}(n s)$ with $s=\sum_{x \in \cX} p_x$, and
$N\mid (n'=m)\sim \mathrm{Mult}\!\left(m;\bar p\right)$ with $\bar p=p/s$.
Therefore,
\begin{align}\label{eq:125B}
\begin{split}
\sup_{p\in\mathcal M_{\epsilon, \zeta}}
\E_{\mathrm{Pois}(np)}\!\left[\,L(p;g_\delta^P,N)\,\right]
= &
\sup_{p\in\mathcal M_{\epsilon, \zeta}} \sum_{m=0}^\infty
\E_{\mathrm{Mult}(m,\bar p)}\!\left[\,L(p;g_{\delta, m},N)\,\right]\P_{\mathrm{Pois}(np)}(n'=m). 
\end{split}
\end{align}
Using \cref{eq:scale,eq:123B} and the fact that $s = \sum_{x \in \cX} p_x\le 1+\epsilon$ for all $p \in\mathcal M_{\epsilon, \zeta}$, we obtain that for all $p\in\mathcal M_{\epsilon, \zeta}$, 
\begin{align}\label{eq:126B} 
\begin{split}
\E_{\mathrm{Mult}(m,\bar p)}\!\left[\,L(p;g_\delta^P,N)\,\right]
\le &\,
(1+\epsilon)\,
\E_{\mathrm{Mult}(m,\bar p)}\!\left[\,L(\bar p;g_{\delta, m},N)\,\right] \\
\le &\, (1+\epsilon)\, \Big( \cR^{\multi}(m, \zeta)+\delta\Big). 
\end{split}
\end{align}
Combine \cref{eq:125B,eq:126B}, we obtain that 
\begin{equation}\label{eq:127B}
\sup_{p\in\mathcal M_{\epsilon, \zeta}}
\E_{\mathrm{Pois}(np)}\!\left[\,L(p;g_\delta^P,N)\,\right]
\le \sup_{p\in\mathcal M_{\epsilon, \zeta}}
\sum_{m=0}^\infty  (1+\epsilon)\left(\cR^{\multi}(m, \zeta)+\delta\right)\P_{\mathrm{Pois}(np)}(n'=m).
\end{equation}
We then upper bound the right-hand-side of Eq.~\eqref{eq:127B}. 
Consider the risk of the trivial all-one estimator $g_1 \equiv 1$, we observe that 
\begin{align}
\label{eq:envelope}
	\cR^{\multi}(m, \zeta) \leq \sup_{ p\in \mathcal M_{\zeta}}\;
\E_{\mathrm{Mult}(m,p)}\!\left[L(p;g_1,N)\right] \leq vd\, (\log \zeta + 1).
\end{align}
Fix $\xi\in(0,1]$ and set $t_p=\lfloor \tfrac{np}{1+\xi}\rfloor$.
Split the sum on the right-hand-side of Eq.~\eqref{eq:127B} at $t_p$ and use monotonicity in $m$ for the large-$m$ block, and the upper bound in \cref{eq:envelope} for the small-$m$ block, we obtain the following:  
\begin{align}\label{eq:128B}
\begin{split}
& \sum_{m=0}^\infty (1+\epsilon)\left(\cR^{\multi}(m, \zeta)+\delta\right)\P_{\mathrm{Pois}(np)}(n'=m) \\
& \le
(1+\epsilon)\Big\{ vd (\log \zeta + 1) \,\P_{\mathrm{Pois}(np)}(n'\le t_p) + \cR^{\multi}(t_p, \zeta)+\delta\Big\}.
\end{split}
\end{align}
By a Poisson left-tail bound, we conclude that for all $p\in\mathcal M_{\epsilon, \zeta}$,
\[
\P_{\mathrm{Pois}(np)}\!\left(n'\le \frac{np}{1+\xi}\right)\le \exp\!\Big(-\frac{\xi^2 n}{24}\Big).
\]
Therefore, combining Eqs.~\eqref{eq:124B}--\eqref{eq:128B}, we see that 
\begin{align}\label{eq:129B}
 \cR^{\pos}(n,\epsilon, \zeta)  \le
(1+\epsilon)\Big\{ vd\, (\log \zeta + 1) \, \,e^{-\xi^2 n/24} +  \cR^{\multi}\big(\lfloor \tfrac{n(1 - \epsilon)}{1+\xi}\rfloor, \zeta \big)  + \delta\Big\}.
\end{align}
We next lower bound $\cR^{\pos}(n,\epsilon, \zeta) $ using the Poisson Bayes risk established in Step~I, which in turn yields a lower bound for $\cR^{\multi}$. 
Recall that $\pi$ is a prior that we introduced in Step I. 
Based on $\pi$, we define a truncated distribution $\bar \pi$ as 
\[
\bar \pi(\d p)=\frac{\pi(\d p)\,\mathbf 1_{\mathcal M_{\epsilon, \zeta}}(p)}{\pi(\mathcal  M_{\epsilon, \zeta})}.
\]
Let $ g_{\bar \pi}$ be a Bayes estimator for the concrete score under prior $\bar \pi$ with data drawn from the Poissonized model (with parameter $ np$). 
By the construction of $\pi$, we must have $g_{\bar\pi} \in [\zeta^{-1}, \zeta]$, otherwise the risk could be reduced by projecting the density ratio estimates onto the interval $[\zeta^{-1}, \zeta]$.
Therefore, for all $p \in \mathcal \mathcal  M_{\epsilon, \zeta}$
\begin{align*}
	L(p; g_{\bar \pi}, N) \leq 2vd \zeta (1 + \log \zeta). 
\end{align*}
Note that $\cR_B^{\pos}(\pi)$ from \cref{eq:poisson-bayes} admits the following decomposition: 
\begin{align}
 &
\cR_B^{\pos}(\pi) \le \E_{p \sim \pi}\Big[\, \E_N\big[\, L(p; g_{\bar \pi},N)\, \big] \Big] =
\int_{\mathcal \mathcal  M_{\epsilon, \zeta}}\E_N\big[ L(p; g_{\bar \pi},N) \big] \pi(\d p)\;+\;
\int_{\mathcal  M_{\epsilon, \zeta}^c}\E_N\big[ L(p; g_{\bar \pi},N) \big]\, \pi(\d p),\nonumber 
\end{align}
where the inner expectations are taken over $N_x \sim \mbox{Pois}(n p_x)$ for all $x \in \mathcal{X}$. 
In addition, note that 
\begin{align*}
	&
\int_{\mathcal  M_{\epsilon, \zeta}}\E_N\big[\, L(p; g_{\bar \pi},N)\, \big] \pi(\d p)
=\pi(\mathcal  M_{\epsilon, \zeta})\!\int_{\mathcal  M_{\epsilon, \zeta}}\E_N\big[\, L(p; g_{\bar \pi},N)\, \big] \bar\pi(\d p),
\notag\\
&
\int_{\mathcal  M_{\epsilon, \zeta}^c}\E_N\big[\, L(p; g_{\bar \pi},N)\, \big] \pi(\d p)
\le 2 vd\, \zeta(\log \zeta + 1)\,\pi\big(\mathcal  M_{\epsilon, \zeta}^c\big).
\end{align*}
Therefore, 
\begin{align}
\label{eq:cRB-pos-lower}
	\cR_B^{\pos}(\pi)\ \le\ \pi(\mathcal  M_{\epsilon, \zeta})\!\int_{\mathcal  M_{\epsilon, \zeta}}\E_N\big[\, L(p; g_{\bar \pi},N)\, \big] \bar\pi(\d p)
\;+\; 2 vd\, \zeta(\log \zeta + 1)\,\pi\big(\mathcal  M_{\epsilon, \zeta}^c\big). 
\end{align}
Recall that $\pi$ is a product distribution, by standard concentration arguments we have 
\[
\pi\big(\mathcal  M_{\epsilon, \zeta}^c\big)\;\le\;2\exp\!\Big(-\tfrac{\epsilon^2|\cX|}{2\eta^2}\Big).
\]
In addition, we note that 
\begin{align}
\label{eq:29}
	\int_{\mathcal  M_{\epsilon, \zeta}}\E_N\big[\, L(p; g_{\bar \pi},N)\, \big] \bar\pi(\d p) \leq \cR^{\pos}(n,\epsilon, \zeta). 
\end{align}
Combining the above bounds, we conclude that 
\begin{align*}
	\cR_B^{\pos}(\pi)\ \le\ \cR^{\pos}(n,\epsilon, \zeta) +  2 vd\, \zeta(\log \zeta + 1)\,\exp\!\Big(-\tfrac{\epsilon^2|\cX|}{2\eta^2}\Big). 
\end{align*}
Putting together the above upper bound an \cref{eq:129B}, we obtain that 
\begin{align*}
	 & \cR_B^{\pos}(\pi) \\
	 & \le
(1+\epsilon)\Big\{ vd\, (\log \zeta + 1) \, \,e^{-\xi^2 n/24} +  \cR^{\multi}\big(\lfloor \tfrac{n(1 - \epsilon)}{1+\xi}\rfloor, \zeta \big)  + \delta\Big\}  + 2 vd\, \zeta(\log \zeta + 1)\,\exp\!\Big(-\tfrac{\epsilon^2|\cX|}{2\eta^2}\Big). 
\end{align*}
We can remove $\delta$ as it is arbitrarily small.
Taking $n\mapsto \lfloor \tfrac{n(1 - \epsilon)}{1+\xi}\rfloor$ and rearranging, we get
\begin{align}
\label{eq:32}
\begin{split}
	& \cR^{\multi}\big(n, \zeta \big) \\
&\ge\;
\frac{1}{1+\epsilon}\, \Big( \cR_B^{\pos}(\pi) - 2 vd\, \zeta(\log \zeta + 1)\,\exp\!\Big(-\tfrac{\epsilon^2|\cX|}{2\eta^2}\Big)\Big) 
\;-\; vd\, (\log \zeta + 1) \, \,e^{-\xi^2 n/24}. 
\end{split}
\end{align}
\begin{revision}
Substituting in the definition of $\eta$ gives 
\begin{align*}
\begin{split}
	& \cR^{\multi}\big(n, \zeta \big) \\
&\ge\;
\frac{1}{1+\epsilon}\,  \cR_B^{\pos}(\pi) - \frac{2}{1+\epsilon} vd\, \zeta(\log \zeta + 1)\,\exp\!\big(-2\epsilon^2 n\big) 
\;-\; vd\, (\log \zeta + 1) \, \,e^{-\xi^2 n/24}. 
\end{split}    
\end{align*}
Letting $\epsilon = 0.49$ and plugging in \cref{eq:step1-result} yields 
\begin{align*}
    \cR^{\multi}\big(n, \zeta \big) \geq & \,\frac{1}{60 \times 1.49} \cdot \frac{|\cX|}{n} - \frac{2}{1.49}vd\, \zeta(\log \zeta + 1)\,e^{-0.4802 n} -\; vd\, (\log \zeta + 1) \, \,e^{-\xi^2 n/24} \\
    \geq & \,\frac{|\cX|}{1000 n}   -\; vd\, (\log \zeta + 1) \, \,e^{-\xi^2 n/24}. 
\end{align*}
Setting 
\begin{align*}
	\xi = \frac{10 \sqrt{\,\log \big(nvd \zeta \big)}}{\sqrt{n}} 
\end{align*}
completes the proof of the theorem.  
\end{revision}

\subsection{Proof of Theorem \ref{thm:thresholding}}
\label{proof:thm:thresholding}

This proof relies on the following standard binomial concentration inequality. 

\begin{lem}
\label{lemma:binomial}
	Let $X \sim \mathrm{Binom}(n, p)$. Then for any $\delta \in (0, 1)$, 
	\begin{align*}
		\P\big(\, |X - np| \geq \delta np \, \big) \leq 2 \exp\left( -\frac{\delta^2 np}{3} \right). 
	\end{align*}
\end{lem}

When $p_x \geq 24\log (n / \delta) / n$, then $N_x \sim \mathrm{Binom}(n, c \log (n / \delta) / n)$ for $c \geq 24$. By Lemma \ref{lemma:binomial}, 
\begin{align*}
	\P(N_x < \log n ) \leq \P \Big(\, \big| N_x - c \log (n / \delta) \big| \geq \frac{c}{2} \log (n / \delta) \,\Big) \leq 2 \exp \left( -\frac{c \log (n / \delta)}{12} \right) \leq \frac{2 \delta^2}{n^2}. 
\end{align*}
For $x \in \mathcal{X}$, we define $\mathcal{E}_x = \{\,p_x \geq 24 \log (n / \delta) / n, \, N_x < \log n\}$. Then for all $x \in \mathcal{X}$, $\mathbb{P}(\mathcal{E}_x) \leq 2\delta^2 / n^2$. 
Observe that 
%
\begin{align*}
	\cL_{\SE}(s, p) = &  \E_{x \sim p} \Big[ \sum_{y \in N(x)} \overrightarrow{Q}(y, x) \Big( s(y, x) - \frac{p_y}{p_x}  - \frac{p_y}{p_x} \log \frac{ s(y, x)}{{p_y / p_x}}\, \Big) \Big] \\
    = & \sum_{x \in \mathcal{X}} \sum_{y \in N(x)} p_x\overrightarrow{Q}(y, x) \Big( s(y, x) - \frac{p_y}{p_x}  - \frac{p_y}{p_x} \log \frac{ s(y, x)}{{p_y / p_x}}\, \Big)  \\
    = & \sum_{x \in \mathcal{X}} \sum_{y \in N(x)} \1_{\cE_x \cup \cE_y} \, p_x\overrightarrow{Q}(y, x) \Big( s(y, x) - \frac{p_y}{p_x}  - \frac{p_y}{p_x} \log \frac{ s(y, x)}{{p_y / p_x}}\, \Big) \\
    & + \sum_{x \in \mathcal{X}} \sum_{y \in N(x)} \1_{\cE_x^c \cap \cE_y^c} \, p_x\overrightarrow{Q}(y, x) \Big( s(y, x) - \frac{p_y}{p_x}  - \frac{p_y}{p_x} \log \frac{ s(y, x)}{{p_y / p_x}}\, \Big). 
\end{align*}
Define $Q_{\infty} = \sup_{x, y} |\overrightarrow{Q}(y, x)|$. 
Note that for $p \in \cP_\zeta$, 
\begin{align}
\label{eq:A2-1}
    \E_N\Big[ \1_{\cE_x \cup \cE_y} \, p_x \overrightarrow{Q}(y, x) \Big( s(y, x) - \frac{p_y}{p_x}  - \frac{p_y}{p_x} \log \frac{ s(y, x)}{{p_y / p_x}}\, \Big)\Big] \leq \frac{4\delta^2 Q_{\infty}}{n^2} \Big( \frac{p_x}{n} - p_y + p_y \log (n\zeta) \Big), 
\end{align}
where the expectation is taken with respect to the counts $\{N_x: x \in \mathcal{X}\}$. 
%

For $x, y \in \cX$ with $y \in N(x)$, if $p_x < 24 \log (n / \delta) / n$ or $p_y < 24 \log (n / \delta) / n$, then by the upper bound $p_y / p_x \leq \zeta$, it holds that $p_y \leq 24 \zeta  \log (n / \delta) / n$. 
In this case, the following statements are true: 
\begin{enumerate}
    \item If $\min \{N_x, N_y\} < \log n$, then
\begin{align}
\label{eq:A2-2}
\begin{split}
	& \1_{\cE_x^c \cap \cE_y^c} \, p_x\overrightarrow{Q}(y, x) \Big( s(y, x) - \frac{p_y}{p_x}  - \frac{p_y}{p_x} \log \frac{ s(y, x)}{{p_y / p_x}}\, \Big) \\
    & =  p_x \,\1_{\cE_x^c \cap \cE_y^c} \, \overrightarrow{Q}(y, x) \Big( \frac{1}{n} - \frac{p_y}{p_x}  + \frac{p_y}{p_x} \log \frac{ n p_y}{{  p_x}}\, \Big) \\
    & \leq \frac{Q_{\infty}}{n} \Big( p_x + 24 \zeta \log (n / \delta) \big( \log n + \log \zeta\big) \Big). 
\end{split}
\end{align}
\item If $\min \{N_x, N_y\} \geq \log n$, then 
\begin{align*}
\begin{split}
    & \1_{\cE_x^c \cap \cE_y^c} \, p_x\overrightarrow{Q}(y, x) \Big( s(y, x) - \frac{p_y}{p_x}  - \frac{p_y}{p_x} \log \frac{ s(y, x)}{{p_y / p_x}}\, \Big) \\
    & = \1_{\cE_x^c \cap \cE_y^c} \, p_x\overrightarrow{Q}(y, x) \Big( \frac{N_y}{N_x} - \frac{p_y}{p_x}  - \frac{p_y}{p_x} \log \frac{ N_y / N_x}{{p_y / p_x}}\, \Big) \\
    & \leq \frac{Q_{\infty}p_x N_y}{N_x} + Q_{\infty}p_y \log \zeta + Q_{\infty}p_y \log |\mathcal{X}|.
\end{split}
\end{align*}
Taking the expectation over the count data $N$, we obtain that  
\begin{align}
\label{eq:A2-4}
\begin{split}
    & \E\Big[ \1_{\cE_x^c \cap \cE_y^c} \, p_x\overrightarrow{Q}(y, x) \Big( s(y, x) - \frac{p_y}{p_x}  - \frac{p_y}{p_x} \log \frac{ s(y, x)}{{p_y / p_x}}\, \Big) \Big| \, N_x \geq \log n \Big] \\
    & \leq \E_N\Big[ \frac{Q_{\infty}np_x p_y}{N_x}  \,\Big| \, N_x \geq \log n\Big] + Q_{\infty} p_y \log \zeta +  Q_{\infty} p_y \log |\mathcal{X}| \\
    & \leq \frac{24 Q_{\infty} \zeta \log (n / \delta)}{n} \E_N \Big[ \frac{np_x}{N_x} \,\Big| \, N_x \geq \log n \Big] + Q_{\infty} p_y \log \zeta +  Q_{\infty} p_y \log |\mathcal{X}| \\
    & \leq \frac{96Q_{\infty}  \zeta \log (n / \delta)}{n} + \frac{24Q_{\infty} \zeta \log (n / \delta) (\log |\mathcal{X}|+ \log \zeta )}{n}.
\end{split}
\end{align}
\end{enumerate}
On the other hand, if both $p_x \geq 24 \log (n / \delta) / n$ and $p_y \geq 24 \log (n / \delta) / n$, then  for any $\varepsilon_x, \varepsilon_y \in(0,1)$, 
\begin{align*}
	& \P\big(|N_x - np_x| \geq \varepsilon_x np_x \big) \leq 2 e^{-\varepsilon_x^2 n p_x / 3}, \\
	& \P\big(|N_y - np_y| \geq \varepsilon_y np_y \big) \leq 2 e^{-\varepsilon_y^2 n p_y / 3}. 
\end{align*}
Taking 
\begin{align*}
	\varepsilon_x = \sqrt{\frac{3 \log(|\cX| / \delta)}{n p_x}}, \qquad \varepsilon_y = \sqrt{\frac{3 \log(|\cX| / \delta)}{n p_y}}, 
\end{align*}
we conclude that with probability at least $1 - \delta$, for all $x, y \in \cX$ with $y \in N(x)$ and $p_x, p_y \geq 24 \log (n / \delta) / n$, it holds that 
\begin{align}
\label{kappa-bound}
	\frac{1-\max\{\varepsilon_x, \varepsilon_y\}}{1+\max\{\varepsilon_x, \varepsilon_y\}}
\;\le\;
\frac{s(y, x)}{p_y / p_x}
\;\le\;
\frac{1+\max\{\varepsilon_x, \varepsilon_y\}}{1-\max\{\varepsilon_x, \varepsilon_y\}} = \kappa_{y, x}. 
\end{align}
Let $h(u)=u-1-\log u$, and note that  $h(\kappa_{y, x})\le\frac{(\kappa_{y, x}-1)^2}{2}=\frac{2\max\{\varepsilon_x, \varepsilon_y\}^2}{(1-\max\{\varepsilon_x, \varepsilon_y\})^2}$.
When \cref{kappa-bound} holds, we have 
\begin{align}
\label{eq:A2-5}
	p_x \Big( s(y, x) - \frac{p_y}{p_x}  - \frac{p_y}{p_x} \log \frac{ s(y, x)}{{p_y / p_x}} \Big)\leq  {8p_y\max \{\varepsilon_x^2, \varepsilon_y^2\}}  \leq \frac{24 \zeta \log(|\cX| / \delta)}{n }. 
\end{align} 
Combining \cref{eq:A2-1,eq:A2-2,eq:A2-4,eq:A2-5}, we conclude that with probability at least $1 - \delta$, we have 
\begin{align*}
	\cL_{\SE}(s, p) = &  \E_{x \sim p} \Big[ \sum_{y \in N(x)} \overrightarrow{Q}(y, x) \Big( s(y, x) - \frac{p_y}{p_x}  - \frac{p_y}{p_x} \log \frac{ s(y, x)}{{p_y / p_x}}\, \Big) \Big] \\
	\lesssim & \frac{ |\cX| \log_v |\cX| \zeta \log (n / \delta) \log (n \zeta)}{n} + \frac{ |\cX| \log_v |\cX| \zeta \log(|\cX| / \delta)}{n }, 
\end{align*}
where ``$\lesssim$'' hides a numerical constant.
The proof is done by integrating out $\delta$.

\subsection{Proof of Lemma \ref{lem:31}}
\label{proof:lem:31}

For the uniform discrete diffusion model, the single-token state space has size $v$, and the
tokenwise generator is
\[
Q_{\rm token}=\frac{1}{v}\mathbf 1 \mathbf 1^\top - I.
\]
Set
\[
P:=\frac1v \mathbf 1 \mathbf 1^\top.
\]
Then $P^2=P$, so $P$ is a projection, and
\[
Q_{\rm token}=P-I=-(I-P).
\]
Hence
\[
e^{tQ_{\rm token}}
=
e^{-t(I-P)}
=
I+\sum_{k=1}^\infty \frac{(-t)^k}{k!}(I-P)^k.
\]
Since $(I-P)^k=I-P$ for all $k\ge 1$, we obtain
\[
e^{tQ_{\rm token}}
=
I+\left(\sum_{k=1}^\infty \frac{(-t)^k}{k!}\right)(I-P)
=
I+(e^{-t}-1)(I-P).
\]
Therefore
\[
e^{tQ_{\rm token}}
=
P+e^{-t}(I-P)
=
e^{-t}I+(1-e^{-t})P.
\]
Writing this entrywise, for $a,b\in [v]$ we have
\[
K_t(a,b):=\bigl(e^{tQ_{\rm token}}\bigr)_{ab}
=
\begin{cases}
\displaystyle e^{-t}+\frac{1-e^{-t}}{v}
=
\frac{1+(v-1)e^{-t}}{v}, & a=b,\\[1.2ex]
\displaystyle \frac{1-e^{-t}}{v}, & a\neq b.
\end{cases}
\]
Let
\[
s_t:=\frac{1+(v-1)e^{-t}}{v},
\qquad
u_t:=\frac{1-e^{-t}}{v}.
\]

Now fix $x\in X$, and let $y\in N(x)$. Under the normalized uniform rate matrix, $x$ and
$y$ differ in exactly one coordinate. Let $i$ be the unique coordinate such that
$x_i\neq y_i$, so that
\[
x=(x_i,x_{-i}), \qquad y=(y_i,x_{-i}).
\]
Since the forward process evolves independently across coordinates,
\[
p_t(x)=\sum_{z\in X} p_0(z)\prod_{j=1}^d K_t(z_j,x_j).
\]
Writing $z=(z_i,z_{-i})$ and grouping according to the value $z_i=a$, we obtain
\[
p_t(x)
=
\sum_{a\in [v]}\sum_{z_{-i}\in [v]^{d-1}}
p_0(a,z_{-i})\,K_t(a,x_i)\prod_{j\neq i} K_t(z_j,x_j).
\]
Define
\[
w_a:=
\sum_{z_{-i}\in [v]^{d-1}}
p_0(a,z_{-i})\prod_{j\neq i} K_t(z_j,x_j).
\]
Then each $w_a\ge 0$, and
\[
p_t(x)=\sum_{a\in [v]} w_a K_t(a,x_i),
\qquad
p_t(y)=\sum_{a\in [v]} w_a K_t(a,y_i).
\]
Set
\[
A:=w_{x_i}, \qquad B:=w_{y_i}, \qquad
R:=\sum_{a\notin\{x_i,y_i\}} w_a.
\]
Using the explicit form of $K_t$, we get
\[
p_t(x)=s_tA+u_tB+u_tR,
\qquad
p_t(y)=u_tA+s_tB+u_tR.
\]
Therefore
\[
\frac{p_t(y)}{p_t(x)}
=
\frac{u_tA+s_tB+u_tR}{s_tA+u_tB+u_tR}.
\]
We claim that this ratio is at most $s_t/u_t$. Indeed,
\[
\frac{u_tA+s_tB+u_tR}{s_tA+u_tB+u_tR}\le \frac{s_t}{u_t}
\]
is equivalent to
\[
u_t(u_tA+s_tB+u_tR)\le s_t(s_tA+u_tB+u_tR),
\]
that is,
\[
u_t^2A+u_ts_tB+u_t^2R\le s_t^2A+s_tu_tB+s_tu_tR.
\]
After canceling the middle terms, this becomes
\[
0\le (s_t^2-u_t^2)A+u_t(s_t-u_t)R,
\]
which is true because $A\ge 0$, $R\ge 0$, and $s_t\ge u_t$.

Hence
\[
\frac{p_t(y)}{p_t(x)}\le \frac{s_t}{u_t}
=
\frac{\frac{1+(v-1)e^{-t}}{v}}{\frac{1-e^{-t}}{v}}
=
\frac{1+(v-1)e^{-t}}{1-e^{-t}}.
\]
Taking the supremum over all $x\in X$ and $y\in N(x)$ yields
\[
\zeta_t\le \frac{1+(v-1)e^{-t}}{1-e^{-t}}.
\]
This completes the proof.

\subsection{Proof of Lemma \ref{lem:32}}
\label{proof:lem:32}

For the absorbing rate matrix, recall that
\[
N(x)=\{y:\overrightarrow{Q}(y,x)>0\},
\]
so $y\in N(x)$ means that $x$ has one masked coordinate and $y$ is obtained by replacing that
mask with a non-mask token.

Fix $x\in X$ and $y\in N(x)$. Then there exist an index $i\in M(x)$ and a token $a\in [K]$
such that
\[
x_i=\mask,\qquad y_i=a,
\]
and $x_j=y_j$ for all $j\neq i$. Write
\[
U:=U(x), \qquad M:=M(x), \qquad u:=|U|,\qquad m:=|M|.
\]
Then
\[
U(y)=U\cup\{i\},\qquad M(y)=M\setminus\{i\},
\]
so that
\[
|U(y)|=u+1,\qquad |M(y)|=m-1.
\]

Let $\alpha_t=e^{-t}$. By the exact forward marginal formula for masking diffusion,
\[
p_t(x)=\alpha_t^u(1-\alpha_t)^m\,\Pr(X_U^0=x_U),
\]
and
\[
p_t(y)=\alpha_t^{u+1}(1-\alpha_t)^{m-1}\,
\Pr(X_{U\cup\{i\}}^0=(x_U,a)).
\]
Therefore
\[
\frac{p_t(y)}{p_t(x)}
=
\frac{\alpha_t}{1-\alpha_t}\cdot
\frac{\Pr(X_{U\cup\{i\}}^0=(x_U,a))}
{\Pr(X_U^0=x_U)}.
\]
The second factor is a conditional probability:
\[
\frac{\Pr(X_{U\cup\{i\}}^0=(x_U,a))}
{\Pr(X_U^0=x_U)}
=
\Pr(X_i^0=a\mid X_U^0=x_U)\le 1.
\]
Hence
\[
\frac{p_t(y)}{p_t(x)}
\le
\frac{\alpha_t}{1-\alpha_t}
=
\frac{e^{-t}}{1-e^{-t}}.
\]
Taking the supremum over all $x\in X$ and $y\in N(x)$ gives
\[
\zeta_t\le \frac{e^{-t}}{1-e^{-t}}.
\]
This completes the proof.

\end{document}